\documentclass{article}

\usepackage{booktabs} % for professional tables
\usepackage{subcaption, graphicx, xcolor, epsfig}
\usepackage{microtype, hyperref, amssymb, amsmath, amsthm, amsfonts, fixmath, flexisym}

\usepackage[in]{fullpage}  

\usepackage{times,color,array,xspace,fancyhdr,comment}

\usepackage{enumitem}
\setlist{nolistsep}
\setlist[itemize]{noitemsep, topsep=0pt}

\usepackage{natbib}

\usepackage{algorithm}
\usepackage[noend]{algorithmic}

\newcommand{\mycaptionof}[2]{\captionof{#1}{#2}}
\newcommand{\noaistats}[1]{}
\newcommand{\SUB}[1]{\ENSURE \hspace{-0.15in} \textbf{#1}}
\newcommand{\algfont}[1]{\texttt{#1}}

\newcommand{\fedua}{\algfont{FedUA}\xspace}
\newcommand{\fedaws}{\algfont{FedAwS}\xspace}
\newcommand{\fedavg}{\algfont{FedAvg}\xspace}

\newcommand{\lepochs}{\ensuremath{E}}  %

\newcommand{\grad}{\triangledown}

\newtheorem{lemma}{Lemma}

\begin{document}

%%%%%%%%%%%%%%%%%%%%%%%%%%%%%%%%%%%%%%%%%%%%%%%%%%%%%%%%%%%%%%%%%%%

\title{Federated Learning of User Authentication Models \vspace{.25cm} }

\author{
Hossein Hosseini\thanks{Qualcomm AI Research, an initiative of Qualcomm Technologies, Inc.}
\and
Sungrack Yun\footnotemark[1]
\and
Hyunsin Park\footnotemark[1]
\and
Christos Louizos\footnotemark[1]
\and
Joseph Soriaga\footnotemark[1]
\and
Max Welling\footnotemark[1]
\and
\texttt{\small \{hhossein,sungrack,hyunsinp,clouizos,jsoriaga,mwelling\}@qti.qualcomm.com}
}

\date{}

%%%%%%%%%%%%%%%%%%%%%%%%%%%%%%%%%%%%%%%%%%%%%%%%%%%%%%%%%%%%%%%%%%%

\maketitle

\begin{abstract}

Machine learning-based User Authentication (UA) models have been widely deployed in smart devices. UA models are trained to map input data of different users to highly separable embedding vectors, which are then used to accept or reject new inputs at test time. 
Training UA models requires having direct access to the raw inputs and embedding vectors of users, both of which are privacy-sensitive information. 
In this paper, we propose Federated User Authentication (\fedua), a framework for privacy-preserving training of UA models. 
\fedua adopts federated learning framework to enable a group of users to jointly train a model without sharing the raw inputs. It also allows users to generate their embeddings as random binary vectors, so that, unlike the existing approach of constructing the spread out embeddings by the server, the embedding vectors are kept private as well. 
We show our method is privacy-preserving, scalable with number of users, and allows new users to be added to training without changing the output layer. Our experimental results on the VoxCeleb dataset for speaker verification shows our method reliably rejects data of unseen users at very high true positive rates. 

\end{abstract}

\section{Introduction}

There has been a recent increase in research and development of User Authentication (UA) models with various modalities such as voice~\citep{snyder2017deep,yun2019end}, face~\citep{wang2018AMS_fv}, fingerprint~\citep{cao2018finger}, or iris~\citep{nguyen2017iris}. Many commercial smart devices such as mobile phones, AI speakers and automotive infotainment systems have adopted machine learning-based UA features for unlocking the system or providing a user-specific service, e.g., music recommendation, schedule notification, or other configuration adjustments. 

User authentication is a decision problem where a test input is accepted or rejected based on its similarity to user's training inputs. The similarity is often computed in an embedding space, i.e., if the predicted embedding of the test input is close to the reference embedding, the input will be accepted, and otherwise rejected. 
Authentication models need to be trained with a large variety of users' data so that the model learns different data characteristics and can reliably reject imposters.
However, due to the privacy-sensitivity of both the raw inputs and the user embeddings, it is not possible to centrally collect users' data to train the model. 
Protecting data privacy is particularly important in UA applications, since the model is likely to be trained and tested in adversarial settings. 
Specifically, leakage of embedding makes the authentication model vulnerable to both training- and inference-time attacks, e.g., poisoning~\citep{biggio2012poisoning} and evasion attacks~\citep{biggio2013evasion,szegedy2013intriguing}.

Federated learning (FL) is a framework for training machine learning models with the local data of users by repeatedly communicating the model weights and gradients between a server and a group of users~\citep{fedavg-mcmahan2017}. FL enables training models without users having to share their data with the server or other users and, hence, is a natural solution for training UA models. Training UA models in the federated setting, however, poses unique challenges described in the following. 

In federated learning of supervised models, typically it is assumed that users have access to pairs of inputs and outputs. In most cases, for any given input, the output is naturally derived from user interactions or can be easily obtained. For example, in the next-word prediction task, the output is simply the next word typed by the user~\citep{hard2018federated}. 
In distributed training of UA models, however, the embeddings are not pre-defined. 
Moreover, even when users know their own embeddings, they need to have access to the embeddings of other users, so that the model can be trained to assign predicted embeddings to be not only close to the reference one, but also far away from other embeddings. 

In this paper, we propose Federated User Authentication (\fedua), a scalable and privacy-preserving framework for training UA models. Our contributions are summarized in the following. 
\begin{itemize}[itemsep=4pt]%,leftmargin=0.5cm
    \item We develop a new approach for UA, where instead of learning the spread out embeddings, users {\it construct} the embeddings with high expected minimum separability. We propose to use random binary vectors, with the length of the vectors being determined by the server such that the minimum distance between embeddings is more than a pre-determined value with high probability. 
    Each user then trains the model to maximize the correlation of the model outputs with their embedding vector. After training, a test input is accepted if the distance of the predicted embedding to the reference one is less than a threshold, and otherwise rejected. We develop a ``warm-up phase" to determine the threshold independently for each user, in which a set of inputs are collected and then the threshold is computed so as to obtain a desired True Positive Rate (TPR). 
    
    \item We show our framework is privacy-preserving and addresses the security problems of existing approaches where embeddings are shared with other users or the server~\citep{yu2020federated}. Moreover, we show that using random binary embeddings enables training the UA models with significantly smaller output size than the number of users, and also allows new users to be added to training after training started without the need to change the output layer. 
    Finally, our method has the advantage that no extra coordination is needed among users or between the users and the server apart from the communications done usually in the FL setting. 
    
    \item We present experimental results of our method on VoxCeleb dataset~\citep{Nagrani2017voxceleb} for speaker verification. We train the models with the speech data of a subset of users ($658$ out of $1,251$ users) and evaluate the authentication performance on the data of the rest of users. 
    We show the models trained in the federated setting achieve high TPR at very low False Positive Rates (FPR) with different lengths of embedding vectors, $n_e$. For example, with a TPR of $80\%$, we obtained FPRs of $0.27\%, 0.18\%$ and $0.16\%$ on data of new users for $n_e=128, 256$ and $512$, respectively.
\end{itemize}

\section{Background}
In this section, we provide a background on training classifiers with Federated Learning (FL) and also machine learning-based User Authentication (UA) models. 

\subsection{Federated Supervised Learning}
Consider a setting where a set $U=\{u_1, \cdots, u_n\}$ of $n$ users want to train a supervised model $F_w$ on their data. 
In FL, a server coordinates with the users to train a model in a privacy-preserving way, i.e., the data of each user will not be shared with the server or other users. Several methods have been proposed for training classifiers in the federated setting~\citep{kairouz2019advances}. The widely-used Federated Averaging framework, also called \fedavg, is described in Algorithm~(\ref{alg:fedavg})~\citep{fedavg-mcmahan2017}.

\begin{algorithm}[t]
\begin{algorithmic}
 \SUB{FedAvg:}
   \STATE {\bf Server:} Initialize $w_0$
   \STATE {\bf Server:} $m \leftarrow \min(c\cdot n, 1)$
   \FOR{each global round $t = 1, 2, \dots$}
     \STATE {\bf Server:} $S_t \leftarrow$ (random set of $m$ users)
     \STATE {\bf Server:} Send $w_{t-1}$ to users $u \in S_t$
     \STATE {\bf Users $u \in S_t$:} $w_{t}^u, n_{s,u} \leftarrow \text{UserUpdate}(w_{t-1}, D_u)$ 
     \STATE {\bf Server:} $w_{t} \leftarrow  \frac{\sum_{u \in S_t} n_{s,u} w_{t}^u}{\sum_{u \in S_t} n_{s,u}}$
   \ENDFOR
   \STATE

 \SUB{UserUpdate($w, D$):}\ \ \  // \emph{Done by users} 
  \STATE $\mathcal{B} \leftarrow$ (split $D$ into batches of size $B$)
  \FOR{each local epoch $i$ from $1$ to $\lepochs$}
    \FOR{batch $b \in \mathcal{B}$}
      \STATE $w \leftarrow w - \eta \grad \ell(w; b)$
    \ENDFOR
 \ENDFOR
 \STATE return $w$ and $|D|$ to server
\end{algorithmic}
\mycaptionof{algorithm}{\citep{fedavg-mcmahan2017} \fedavg. 
  $n$ is the number of users, $c$ is the fraction of users selected for each  round, and $D_u$ is the dataset of user $u$ with $n_{s,u}$ samples.}\label{alg:fedavg}
\end{algorithm}

\subsection{User Authentication with Machine Learning}
User authentication is a decision problem where a test input is accepted (reference user) or rejected (imposter user) based on the characteristics of  input data. The authentication is done by comparing an error value with a threshold $\tau$ as:
\begin{eqnarray}\label{eq:sim_score}
d({\bf x}_{\texttt{ref}}, x') \underset{\tiny \texttt{accept}}{\overset{\tiny \texttt{reject}}{\gtrless}} \tau,
\end{eqnarray}
where $d$ is a distance function, ${\bf x}_{\texttt{ref}}$ is the set of training inputs and $x'$ is the test sample. 

The distance is usually computed in an embedding space. Let ${\bf x}=\{x_{i}\}$ and ${\bf y}=\{y_{i}\}, i\in\{1,...,n\},$ be the set of inputs and embedding vectors, respectively, where ${x}_{i}=\{x_{ij}\}, j\in\{1,...,n_{s,i}\}$, and $n_{s,i}$ is the number of training samples of user $i$. 
The UA model $F_w$ with parameters $w$ is trained to minimize the distance of the output of the model on $x_{ij}$ with the embedding vector $y_i$, and maximize the distance to other embeddings $y_k, k\neq i$. 
The loss function is defined as follows:
\begin{align}\label{eq:loss}
&\ell({\bf x},{\bf y};w)=\sum_i {\ell(x_{i},y_i;w)}, \text{ where} \hspace{0.2cm} \ell(x_{i},y_i;w)= \frac{1}{n_{s,i}}\sum_j{\bigg(d(y_i,F_{w}(x_{ij}))-\lambda \sum_{k\notin i}{d(y_k,F_{w}(x_{ij}))}\bigg)}.
\end{align}
At test time and for user $i$, a sample $x'$ is accepted if $d(y_i,F_w(x'))\leq \tau$.     
\section{User Authentication with Federated Learning}\label{sec:fedUA}
In this section, we first outline the privacy requirements of UA applications and then review the challenges of training UA models in the federated setting. 

\subsection{Problem Statement}\label{sec:problem_statement}
Authentication models need to be trained with a large variety of users' data so that the model learns different data characteristics and can reliably authenticate users. For example, speaker recognition models need to be trained with the speech data of users with different ages, genders, accents, etc., to be able to reject imposters with high accuracy. 
One approach for training UA models is that a server collects the data of the users and trains the model centrally. This approach, however, is not privacy-preserving due to the need of having direct access to the personal data of the users. Protecting data privacy is particularly important in UA applications, where the model is likely to be trained and tested in adversarial settings. 

In UA models, both the raw inputs and the embedding vectors are considered sensitive information. 
Specifically, sharing the raw inputs with the server, aside from exposing the user's identity, e.g., voice or face attributes, makes the model vulnerable to test-time attacks, e.g., by authenticating copies of the original inputs. 
The embedding vector also needs to be kept private since it is used to authenticate a user. Leakage of the embedding vector makes the authentication model vulnerable to both training- and test-time attacks as explained in the following. 
\begin{itemize}[itemsep=4pt]
    \item {\bf Poisoning attack}~\citep{biggio2012poisoning}: The server, in addition to the users' data, trains the model with data $(x^*, y_t)$ for target user $u_t$. At test time, the model outputs $y_t$ when queried with $x^*$ and thus wrongly authenticates $x^*$ as a true sample from user $u_t$. 
    \item {\bf Evasion attack}~\citep{biggio2013evasion,szegedy2013intriguing}: Attacks based on adversarial examples are known to be highly effective against deep neural networks~\citep{carlini2017towards}. In the context of UA models, when a target embedding vector is known, an evasion attack can be performed to slightly perturb any given input such that the predicted embedding matches a target embedding and thus is accepted by the model. 
\end{itemize}

\subsection{Challenges}
An alternative approach is using the FL framework, which enables training with data of a large number of users while keeping their data private by design. Training UA models in the FL setting, however, poses its own challenges described in the following. 

\noindent{\bf Problem (1).} 
In distributed training of UA models, the embedding vectors of users are not pre-defined. One approach to define embeddings is that the server assigns a unique ID to each user. Thus, user $i$ trains the model with pairs of $(x_{ij},U_i)$, where $U_i$ is the corresponding one-hot representation of the user ID. This approach, however, has the following drawbacks:
\begin{itemize}[itemsep=4pt]
    \item It is not privacy-preserving as the server knows the embedding vectors of users, which makes the model vulnerable against both training- and test-time attacks.
    \item It is not scalable because the size of the network output will be equal to the number of users. This is a major drawback especially in the FL setting because 1) model weights and gradients must be communicated many times between the server and the users, and 2) training and inference are usually done on resource-constrained local devices. 
    \item The number of participants needs to be known beforehand. In typical FL settings, new users might join after training starts, hence the model design must allow for various numbers of users. However, with one-hot output encoding, the output length must be set before training and cannot be increased after training starts. 
\end{itemize}

\noindent{\bf Problem (2).} 
Even when each user knows its own embedding, they need to have access to embedding vectors of other users in order to train the model with the loss function defined in Equ.~(\ref{eq:loss}). Due to privacy constraints, however, embeddings cannot be shared with other users or the server.

\subsection{Related work: Federated Averaging with Spreadout (\fedaws)}
The loss function defined in Equation~(\ref{eq:loss}) causes the UA model to cluster training data such that the data of each user are placed near its corresponding embedding and far away from other embeddings. A recent paper~\citep{yu2020federated} observed that, alternatively, the model could be trained to maximize the pairwise distances between different embeddings. They then proposed Federated Averaging with Spreadout (\fedaws) framework, where the server, in addition to federated averaging, performs an  optimization step on the embedding vectors to ensure that different embeddings are separated from each other by at least a margin of $\nu$. In particular, in each round of training, the server applies the following geometric regularization:
\begin{align}
\mathrm{reg}_{\mathrm{sp}}({\bf y}) = \sum_{u\in \{1,\cdots,n\}} {\sum_{u\neq u'}{(\max(0,\nu-d(y_u,y_{u'})))^2}}.
\end{align}

\fedaws solves the problem of sharing embedding vectors with other users, but still requires sharing embeddings with the server, which seriously undermines the performance of UA models in adversarial settings, specifically against poisoning and evasion attacks as explained in~\ref{sec:problem_statement}. 
Hence, the question is how can we maximize the pairwise distances between embeddings in a privacy-preserving way? In the next section, we present our framework for addressing the challenges of training UA models in the federated setting. 
         
\section{Proposed Method}

There are two main requirements for training UA models, 1) from the performance perspective, the embedding vectors must be highly separable~\citep{yu2020federated}, and 2) the training method must be privacy-preserving, i.e., the raw inputs or the embedding vector may not be shared with other entities participating in training. In the following, we present Federated User Authentication (\fedua) framework and describe its properties. 

\subsection{Federated User Authentication (\fedua)}
We adopt the FL framework for training UA models since it is a natural choice for training machine learning models on the data of a large number of users without having direct access to the raw inputs. Moreover, in our proposal, users train the model with random embedding vectors generated prior to the training. 
We show the proposed method is privacy-preserving and also provides a high degree of separability between the embedding vectors. 

\noindent{\bf Training.} 
Let $n_e$ be the length of the embedding vector. We propose to use random binary vectors as embeddings, i.e., $y_{ik}\sim \mathrm{Ber}(p=0.5)$, where $y_{ik}$ is the $k$-th element of the embedding vector of user $i$ and $\mathrm{Ber}(p)$ is a Bernoulli distribution with probability $p$. In experiments, we observed that a Bernoulli distribution performs better than other choices of generating random vectors. 
A model with $n_e$ binary outputs can be interpreted as an ensemble of $n_e$ binary classifiers where each classifier independently splits users into roughly two equal groups. The length of the embedding vector is determined by the server such that the generated random vectors are sufficiently separable. We provide probabilistic lower bounds on the minimum distance of the embedding vectors of length $n_e$ as a function of the number of users $n$.

For training, the output vector of the model is passed through a sigmoid layer to generate predicted embeddings $\hat{y}$ in the range of $[0,1]$. The user $i$ trains the model to maximize the correlation of the predicted and true embeddings using the following loss function:
\begin{align}\label{eq:loss_bin}
&\ell({x_i},{y_i};w)=-\frac{1}{n_{s,i}}(2y_i-1)^T \sum_j F_w(x_{ij}).
\end{align}
The loss function is designed so as to encourage the predicted embedding vector to be high where $y_{ij}=1$ and similarly to be low where $y_{ij}=0$.
% {\\ \color{blue}sungrack: in the equation, $T$ is defined before?}

\noindent{\bf Authentication.} 
After training, each user deploys the model as a binary classifier to accept or reject a test sample. For an input $x'$, the authentication is done by comparing the distance of the predicted embedding $\hat{y}=F(x')$ to the reference embedding $y$ with a threshold $\tau$ as follows:
\begin{align}
d(y, \hat{y}) \underset{\tiny \texttt{accept}}{\overset{\tiny \texttt{reject}}{\gtrless}} \tau,
\end{align}
where $d(y, \hat{y})= \|y-\hat{y}\|_2^2$.
The threshold is determined by each user separately in a “warm-up phase,” such that the True Positive Rate (TPR) is more than a value, say $q=90\%$. The TPR is defined as the rate that the reference user is correctly authenticated. In the warm-up phase, $k$ user inputs $x'_j, j=\{1,\cdots,k\}$, are collected and the corresponding distances to reference embedding, $d_j= \|y-F(x'_j)\|_2^2$, are computed. The threshold is then set such that a desired fraction $q$ of inputs are authenticated. 
Our proposed \fedua framework is described in Algorithm~(\ref{alg:fedua}).

\begin{algorithm}[t]
\begin{algorithmic}
 \SUB{Training:}
   \STATE {\bf Server:} Determine length of embedding vectors, $n_e$.
   \STATE {\bf Server:} Send $n_e$ to all users
   \STATE {\bf Each user:} Generate a random binary vector of length $n_e$ as embedding vector
   \STATE {\bf Server and users:} $F\leftarrow$ \algfont{FedAvg}$(D_u)$, $u\in\{1,\cdots,n\}$ 
   \STATE Return $F$
   \STATE

 \SUB{WarmUpPhase($F, y, r$):}\ \ \  // \emph{Done by users} 
  \STATE Collect inputs $x'_j, j\in\{1,\cdots,k\}$
  \STATE Compute $d_j= \|y-F(x'_j)\|_2^2, j\in\{1,\cdots,k\}$
  \STATE Set $\tau$ equal to the $i$-th smallest value in $d$ where $i=\lfloor k\cdot r \rfloor$
  \STATE Return $\tau$
  \STATE
  
 \SUB{Authentication($F, y, \tau, x'$):}\ \ \  // \emph{Done by users} 
  \STATE $e\leftarrow \|y-F(x')\|_2^2$
  \IF{$e\leq \tau$} 
    \STATE Return \textsc{Accept}
  \ELSE 
    \STATE Return \textsc{Reject}
  \ENDIF
  
\end{algorithmic}
\mycaptionof{algorithm} {Federated User Authentication (\fedua). 
  $n$ is the number of users, $D_u$ is the dataset of user $u$, $y$ is the reference embedding, $F$ is the trained model, $r$ is the TPR, $x'$ is a test sample, and \fedavg is described in Algorithm~(\ref{alg:fedavg}).}\label{alg:fedua}
\end{algorithm}

\subsection{Analysis of \fedua}

\noindent{\bf Minimum distance between embeddings.} 
The following Lemma provides a probabilistic lower bound on $d_{\min}$. 

\begin{lemma}
Let $n$ be the number of users and $n_e$ be the length of the embeddings. Let also $d_{\min}$ be the minimum Hamming distance between all embedding vectors. We have:
\begin{align}\label{eq:LB}
\Pr(d_{\min} \geq \tau) \geq \Pi_{k=0}^{n-1}(1-\frac{k\cdot V_{\tau}}{2^{n_e}}),
\end{align}
where $V_{\tau}=\sum_{d=0}^{\tau-1}{n_e \choose d}$. 
\end{lemma}
\begin{proof}
Note that $V_{\tau}$ is the number of vectors with distance less than $\tau$ to a given vector. We prove the lemma by induction. For $n=1$, Equ.~(\ref{eq:LB}) trivially holds. Assume it also holds for $m$. The probability that a new vector can be added such that $d_{\min}$ will not decrease is $1-\frac{\hat{V}}{2^{n_e}}$, where $\hat{V}$ is the space occupied by Hamming spheres with radius $\tau-1$ of previous $m$ vectors. 
We have $\hat{V} \leq m\cdot V_{\tau}$, where $m\cdot V_{\tau}$ is the overall occupied space assuming that previous Hamming spheres are disjoint. This completes the proof.
\end{proof}

Using Equ.~(\ref{eq:LB}), for a given number of users and desired minimum distance $d_{\min}$, the server can obtain $n_e$ such that the minimum distance of random embedding vectors is more than $\tau$ with probability of at least $q$.

\noindent{\bf Practical advantages.} 
Using random binary embeddings enables training UA models with significantly smaller output size compared to the one-hot encoding and thus scales to a larger number of users. Moreover, our framework allows new users to be added to training after training started without the need to change the output layer. Although adding new users causes the effective pairwise minimum distance between embedding vectors to decrease, it helps the model performance by training with data of more users. Furthermore, the drop in minimum distance will not be significant either. As an example, assume $n_e=512$ and $n=1000, 2000$ and $5000$. The median values of $d_{\min}$ over $10$ experiments are $d_{\min}=202, 199$ and $196$, respectively, implying that even doubling the of number users only slightly reduces the minimum distance. 
Finally, generating embeddings randomly does not need any coordination among the users or between the users and the server, apart from the communications done usually in the FL setting.

\noindent{\bf Security analysis.} 
In our proposed framework, neither raw inputs nor the embeddings will be shared with the server or other users, which makes the model robust against poisoning and evasion attacks. 
Our method, however, inherits potential privacy leakage of FL methods, where users' data might be recovered from a trained model or the gradients~\citep{melis2019exploiting}. It has been suggested that adding noise to gradients or using secure aggregation methods improve the privacy of FL~\citep{mcmahan2017learning,bonawitz2017practical}. Such approaches can be applied to our framework too. 
       
\section{Related Work}
FL has been used in a variety of applications, such as mobile keyboard prediction~\citep{hard2018federated,yang2018applied}, keyword detection~\citep{leroy2019federated}, medical applications~\citep{brisimi2018federated} and wireless communications~\citep{niknam2019federated}. 
Apple has also said to use FL for the vocal classifier for ``Hey Siri"~\citep{Apple2019}, the details of which however have not been published. 
To the best of our knowledge, our work is the first to explore using FL for privacy-preserving training of UA models. 

Our approach of assigning a random binary vector to each user is related to distributed output representation~\citep{sejnowski1987parallel}, where a binary function is learned for each bit position. It follows~\citep{hinton1986learning} in that functions are chosen to be meaningful and independent, so that each combination of concepts can be represented by a unique representation. Another related method is distributed output coding~\citep{dietterich1991error,dietterich1994solving}, which uses error-correcting codes (ECCs) to improve the generalization performance of classifiers, with the codes constructed such that the length of codewords is greater or equal to number of classes. We, however, use random binary vectors as user embeddings to enable privacy-preserving training of UA models in the federated setting. Moreover, we propose to generate vectors of length much smaller than the number of users to improve the scalability of the method to large number of users.

\section{Experimental Results}
In this section, we first describe the dataset, network and the training setup, and then provide the authentication results of UA models trained in federated setting. 

\subsection{Experimental Setup}

\noindent{\bf Dataset.} 
We evaluate the proposed FedUA framework on the VoxCeleb dataset~\citep{Nagrani2017voxceleb} which is created for large scale text-independent speaker identification in real environments. The dataset contains $1,251$ speakers' data with $45$ to $250$ number of utterances per speaker, which are generated from YouTude videos recorded in various acoustic environments. 

For training UA models, usually only few samples are collected in one setting and the same environment. Hence, we selected speakers that had at least $20$ samples from a single video, which resulted in $658$ speakers. We used the first $2$ seconds of each audio file for training and validation and used the next $2$ seconds for testing the authentication performance of the model on users who participated in training. For each user, we train the model with $15$ utterances and use the remaining $5$ utterances for validation. 
We also generated a dataset of users that we did not select for training by choosing $10$ utterances from remaining speakers and cropping their first $2$ seconds. 
All $2$-second audio files are downsampled by a factor of $2$ to obtain vectors of length $2^{14}$ for model input. 

\noindent{\bf Network architecture.} 
The network consists of three convolutional blocks, composed of convolution, relu, average pooling and Group Normalization (GN) layers, followed by fully-connected and sigmoid layers. GN is used instead of batch-normalization (BN) following the observations that BN does not work well in non-iid data setting similar to our case~\citep{hsieh2019non}. Table~(\ref{tab:model}) provides the details of the network architecture. 

\begin{table}[t]
\caption{Network architecture for UA model trained with speech data. conv1d $(C,K)$ is one-dimensional convolutional layer with $C$ output channels and kernel size of $K$, avg\_pool1d $(r)$ is one-dimensional average pooling with downsampling rate of $r$, GN $(G)$ is group normalization layer with $G$ groups, FC $(n_1,n_2)$ is fully-connected layer with $n_1$ inputs and $n_2$ outputs, and $n_e$ is the length of the embedding vector.}\label{tab:model}
\centering
\begin{tabular}{ |l|c| } 
\hline
{\bf \small Layer} & {\bf \scriptsize Output Size} \\
\hline
{\small Input} & $1\times 2^{14}$ \\
\hline
{\small conv1d $(2^6, 21)$, relu, avg\_pool1d $(2^3)$, GN $(2)$} & $2^6\times 2^{11}$ \\
\hline
{\small conv1d $(2^8, 11)$, relu, avg\_pool1d $(2^5)$, GN $(2)$} & $2^8\times 2^{6}$ \\
\hline
{\small conv1d $(2^{10}, 5)$, relu, avg\_pool1d $(2^6)$, GN $(2)$} & $2^{10}\times 1$ \\
\hline
{\small Flatten} & $2^{10}$\\
\hline
{\small FC $(2^{10}, n_e)$} & $n_e$ \\
\hline
{\small sigmoid} & $n_e$ \\
\hline
\end{tabular}
\end{table}

\noindent{\bf Training setup.} 
We train federated models with \fedavg method with $E=1$ local epoch and fraction $c=5e-3$. 
The models are trained with SGD optimizer with learning rate of $2e-3$. 
We provide experimental results with random embeddings of lengths of $n_e=128, 256$ and $512$, for which the corresponding minimum distances between embedding vectors are $39$, $93$ and $204$, respectively. 

\subsection{Authentication Results}
We provide the experimental results for models trained with random binary embeddings in the federated setting. 
The authentication performance is evaluated on different data, namely 1) training data, 2) validation data of users who participated in training, and 3) data of users who did not participate in training. 
Figure~\ref{fig:results} shows the ROC curves. 
As expected, the authentication performance is best on training data. The performance slightly degrades when the model is evaluated on validation data of users who participated in training and further reduces on data of new users. The models, however, achieve notably high TPR at very low FPRs in all case. For example, with a TPR of $80\%$, we obtained FPRs of $0.27\%, 0.18\%$ and $0.16\%$ on data of new users for $n_e=128, 256$ and $512$, respectively, implying that the model can reliably reject the data of unseen users. Also, as expected, by increasing the length of the embedding vector the performance improves.

\begin{figure*}[t]
\centering
\begin{subfigure}[b]{0.31\linewidth}
  \centering
  \includegraphics[width=\linewidth]{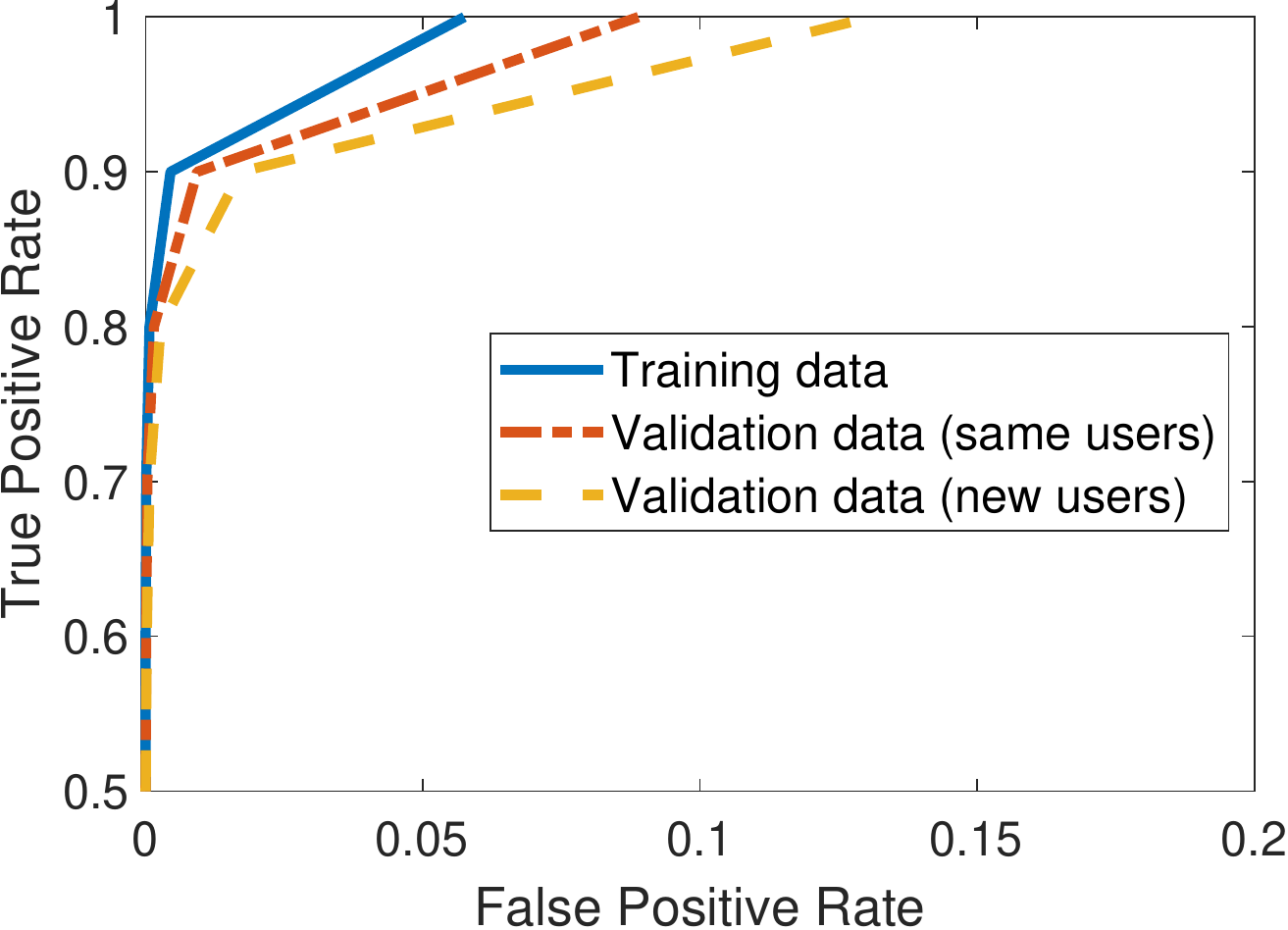}
  \caption{$n_e=128$}
 \end{subfigure}\hspace{0.3cm}
 \begin{subfigure}[b]{0.31\linewidth}
  \centering
  \includegraphics[width=\linewidth]{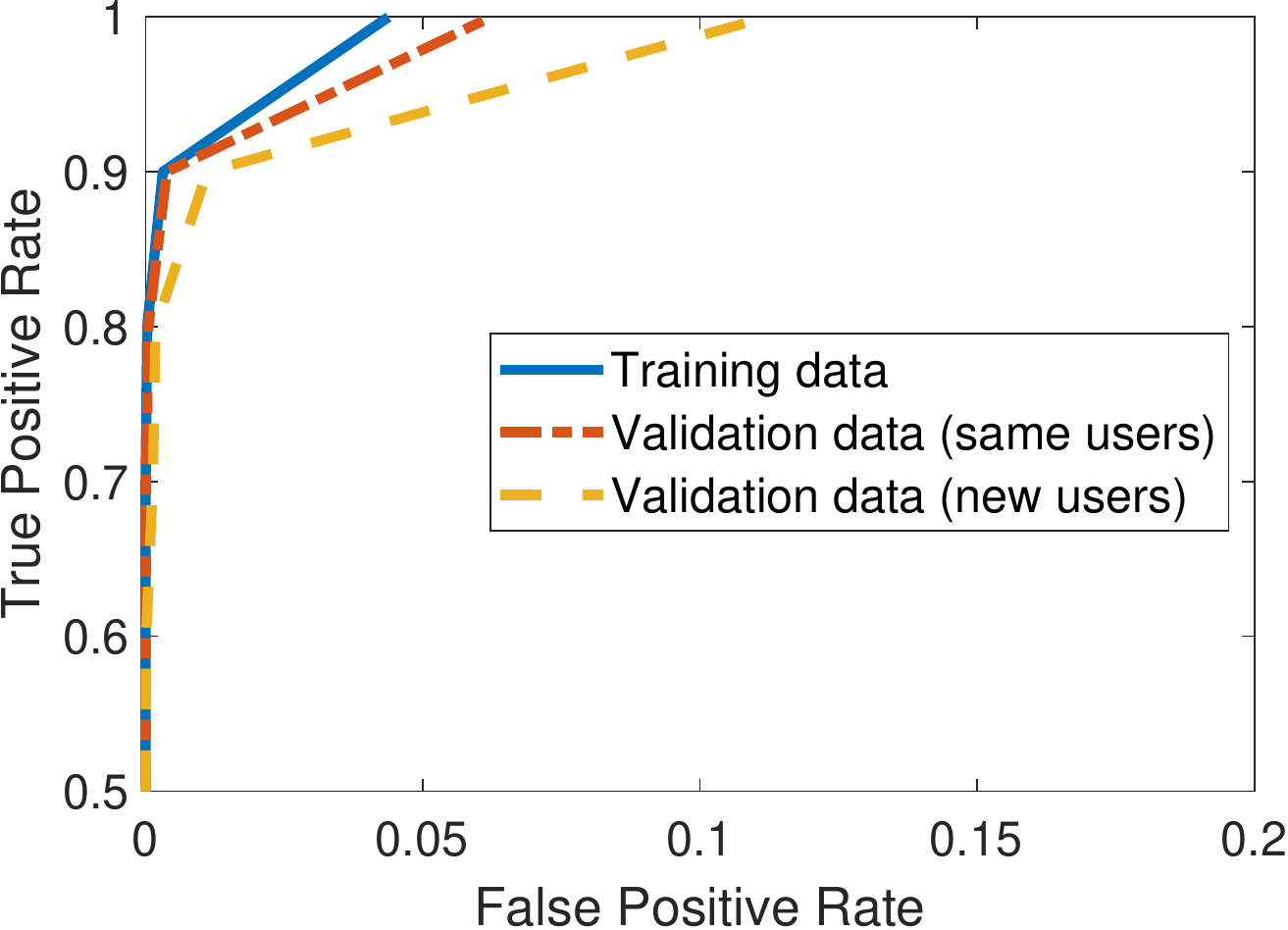}
  \caption{$n_e=256$}
 \end{subfigure}\hspace{0.3cm}
 \begin{subfigure}[b]{0.31\linewidth}
 \centering
  \includegraphics[width=\linewidth]{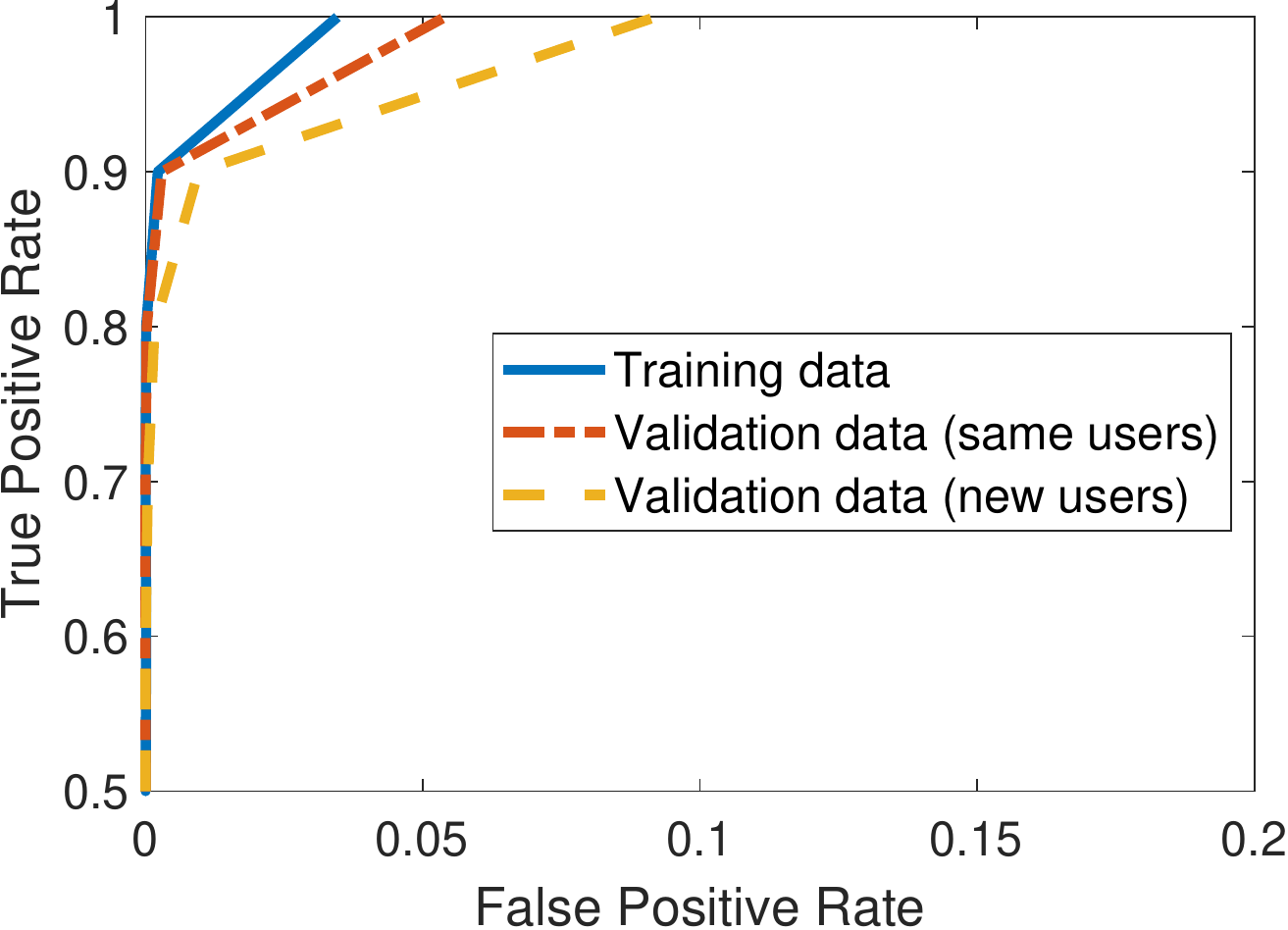}
  \caption{$n_e=512$}
 \end{subfigure}\vspace{-0.2cm}
 \caption{\small ROC curves for models trained with random binary embeddings in federated setting. The models are trained with data of $658$ users (out of $1,251$ users) of the VoxCeleb dataset~\citep{Nagrani2017voxceleb}. The figures show TPR vs FPR for embedding vectors with different lengths of $n_e=128, 256$ and $512$. The authentication performance of the models are evaluated on different data, namely 1) training data, 2) validation data of users who participated in training, and 3) data of users who did not participate in training. As can be seen, the models achieve high TPR at very low FPRs in all case. For instance, with TPR=$80\%$, we obtained FPR=$0.27\%, 0.18\%$ and $0.16\%$ on data of new users for $n_e=128, 256$ and $512$, respectively. Also, as expected, by increasing the length of the embedding vector the performance improves.}%\vspace{-0.30cm} 
\label{fig:results}
\end{figure*}

\vspace{-0.15cm}
\section{Conclusion}
In this paper, we presented \fedua, a framework for training user authentication models. The proposed framework adopts federated learning and random binary embeddings to protect the privacy of raw inputs and embedding vectors, respectively. 
We showed our method is scalable with the number of users and does not need any coordination among the users or between the users and the server, apart from the communications done usually in the FL setting. 
Our experimental results on a speaker verification dataset shows the proposed method reliably rejects data of unseen users at very high true positive rates.

The proposed approach of choosing fixed random binary vectors enables training the model with highly separable embeddings, but does not take into account the characteristics of users' data, e.g., age, gender or accent in speech inputs. In future work, we plan to extend the proposed method to adaptively update the embedding vectors during the training in a privacy-preserving way.

%%%%%%%%%%%%%%%%%%%%%%%%%%%%%%%%%%%%%%%%%%%%%%%%%%%%%%%%%%%%%%%%%%%

{\small 
\bibliography{main}
}
\bibliographystyle{icml20}

%%%%%%%%%%%%%%%%%%%%%%%%%%%%%%%%%%%%%%%%%%%%%%%%%%%%%%%%%%%%%%%%%%%

\end{document}